\documentclass[11pt,fleqn]{amsart} 

\usepackage[usenames,dvipsnames]{xcolor}

\usepackage{amsthm}
\usepackage{amsfonts}
\usepackage[english]{babel}
\usepackage[usenames]{xcolor}
\usepackage{graphicx}
\usepackage{soul}
\usepackage{stfloats}
\usepackage{morefloats}
\usepackage{cite}
\usepackage{lscape}
\usepackage{epstopdf}
\usepackage{braket}
\usepackage[lite]{amsrefs}
\usepackage{mathbbol}
\usepackage{tikz,tikz-cd}
\usepackage{shuffle}

\usepackage{algorithm,algorithmicx,algpseudocode}

\usepackage{amsmath,amstext,amsopn,amsfonts,eucal,amssymb}
\usepackage{graphicx,wrapfig,url}

\newtheorem{theorem}{Theorem}[section]

\newtheorem{lemma}[theorem]{Lemma}
\newtheorem{proposition}[theorem]{Proposition}

\newtheorem{remark}[theorem]{Remark}

\begin{document}

	\title[Leray-Schauder]{Leray-Schauder Mappings for Operator Learning}

	\author{Emanuele Zappala} 
	\address{Department of Mathematics and Statistics, Idaho State University\\
		Physical Science Complex |  921 S. 8th Ave., Stop 8085 | Pocatello, ID 83209} 
	\email{emanuelezappala@isu.edu}

	\maketitle

	\begin{abstract}
		We present an algorithm for learning operators between Banach spaces, based on the use of Leray-Schauder mappings to learn a finite-dimensional approximation of compact subspaces. We show that the resulting method is a universal approximator of (possibly nonlinear) operators. We demonstrate the efficiency of the approach on two benchmark datasets showing it achieves results comparable to state of the art models.
	\end{abstract}


	\date{\empty}


	\section{Introduction}
	
	Operator learning is a branch of deep learning involved with approximating (potentially highly nonlinear) continuous operators between Banach spaces. The interest of operator learning lies in the fact that it allows to model complex phenomena, e.g. dynamical systems, whose underlying governing equations are not known \cite{DeepOnet,ANIE,NIDE}. The study of operator learning was initiated by the theoretical work \cite{chen}, whose implementation was given in \cite{DeepOnet}, and whose error bounds were further developed in \cite{LMK}. 
	
	The problem of operator learning is therefore to model mappings between infinite dimensional spaces. Since most of the algorithms in practice discretize the domains of the functions, and therefore create a finite space upon which the neural networks are used, one might ask whether the mapping is really between function spaces, or it can just be reduced to some map between high dimensional spaces. For instance, one might ask whether it is possible to upsample the discretized domains after training, e.g. to interpolate and output arbitrarily dense predictions, which would show that the model does not depend on the spacetime stamps used during training, performing an interpolation task. In other words, to what extent the model operates in the infinite dimensional spaces, rather than on some high dimensional Euclidean space obtained via discretization. 
	An example where learning on discretized spaces shows its limitations was pointed out in \cite{cont_transformer}, where the tokenization process in the transformer architecture resulted in the model learning step functions. In other words, the discretization of the input resulted in the discretization of the output, and poor interpolation results during test. 
	
	Various models in operator learning deal with such a question in different ways. For intance, \cite{DeepOnet} fixes points in the domain, but allows continuous inputs. The articles \cite{NODE,ANIE,NIDE} use solvers (for ODEs, IEs or IDEs) to allow continuous sampling/output of points. In \cite{FNO}, the use of Fourier transforms allows to upsample the domain of space of functions. In \cite{cont_transformer}, the use of Sobolev-type norms to limit sharp edges due to steep derivatives is employed to regularize the step-like behavior induced by tokenization.
	
	Our perspective in the present work stems from the idea of learning an operator between bases of functions as in \cite{Spectral}, where the Chebyshev polynomials have been used for projections as in the Galerkin methods \cite{Fle}. In fact, it was proved in \cite{EZ_proj} that using Leray-Schauder mappings to ``nonlinearly project'' over a basis of chosen elements approximating a compact set of a Banach space, one can approximate any (possibly nonlinear) continuous operator. This addresses the aforementioned question of learning operators between infinite dimensional spaces rather than just discretization of them, because the approximation happens at the functional level. 
	
	In fact, in the present article we introduce a family of neural networks that approximate a compact subset of a Banach space obtained via Leray-Schauder mappings. Then, the operator learning problem is formulated between the coordinate spaces of the neural networks. This allows us to treat the problem without any discretization, but by using the finite-dimensional reduction provided by the Leray-Schauder mappings to approximate the original learning problem. We show that there exist theoretical guarantees for the approximation to be arbitrarily accurate. In practice, therefore, in this article we also learn the elements upon which the Leray-Schauder mappings project the compact subspace. This represents a paradigmatic shift with respect to the setting of \cite{Spectral}, where the projection used is fixed. 
	
	We numerically demonstrate the efficacy of the  model,  and show that it obtains results comparable to state of the art models. We test the model on interpolation tasks, to demonstrate the independence of the model on the grid size used during training. The systems used for comparison consist of a dataset of integral equations, which shows the efficacy of the model on a task of learning nonlocal operators, and a dataset of Burger's equation, which explores a PDE problem.
	The codes for the implementation of this work are found at \url{https://github.com/emazap7/Leray_Schauder_neural_net}.
	
	The article is organized as follows. In Section~\ref{sec:Leray_Schauder} we give some preliminary results and definitions regarding Leray-Schauder mappings and approximations of nonlinear operators. We also prove that it is possible to construct a universal approximator model by using Leray-Schauder mappings and neural networks to model the finite dimensional subspaces obtained via the leray-Schauder mappings. In Section~\ref{sec:Algo} we discuss the details of the deep learning algorithm which implements this theoretical framework. In Section~\ref{sec:exp}, we provide numerical experiments for the algorithms, and compare with other operator learning methods.

	\section{Theoretical Construction}\label{sec:Leray_Schauder}
	
	We recall the following result obtained in \cite{EZ_proj}, which gives a method for approximating a continuous (nonlinear) operator between Banach spaces over a compact. The core of the approach is that over a compact it is possible to approximate the domain with finitely many elements, which allows to reformulate the approximation problem in finite dimensions. For this purpose, the Leray-Schauder mappings are used. 
	
	\begin{theorem}\label{thm:Universal}
		Let $X$ and $Y$ be Banach spaces, let $T: X\longrightarrow Y$ be a continuous (possibly nonlinear) map, and let $K\subset X$ be a compact subset. Then, for any choice of $\epsilon > 0$ there exist natural numbers $n,m\in \mathbb N$, finite dimensional subspaces $E_n \subset X$ and $E_m \subset Y$, a continuous map $P_n : X\longrightarrow E_n$, and a neural network $f_{n,m} : \mathbb R^n \longrightarrow \mathbb R^m$ such that for every $x\in K$
		\begin{eqnarray}\label{eqn:LS}
		||T(x) - \phi_m^{-1}f_{n,m}\phi_nP_n(x)|| < \epsilon,
		\end{eqnarray}
		where $\phi_k : E_k \longrightarrow \mathbb R^k$ indicates an isomorphism between the finite dimensional space $E_k$ and $\mathbb R^k$.  
	\end{theorem}

		As pointed out in \cite{EZ_proj}, the operators $P_n$ and $P_m$ have an explicit and very simple construction. Following \cite{Topological}, let $K$ be a compact, $E_n$ an $n$-dimensional space spanned by the elements of an $\epsilon$-net $\{x_1, \ldots, x_n\}$, where $\epsilon >0$ is chosen and fixed. Then $P_n : K \longrightarrow E_n$ is defined by the assignment
		\begin{eqnarray}
			P_n x = \frac{\sum_{i=1}^{n} \mu_i(x) x_i}{\sum_{j=1}^n\mu_j(x)},
		\end{eqnarray}
		where 
		\begin{eqnarray}\label{eqn:mui}
		\mu_i(x) = \begin{cases}
		\epsilon - \|x-x_i\|, \hspace{1cm} \|x-x_i\| \leq \epsilon\\
		0, \hspace{2.855cm} \|x-x_i\| > \epsilon
		\end{cases}
		\end{eqnarray}
		for all $i=1, \ldots , n$. We observe that we allow the case where $x_i$ might not be in $K$. 
			
		We will refer to these operators as {\it Leray-Schauder maps}, or {\it Leray-Schauder projections}, although they are not linear. As shown in \cite{EZ_proj}, one can choose the elements $x_i$ in such a way that $P_n(x_i) = x_i$, which would resemble a projection. 
		
		The construction of Theorem~\ref{thm:Universal} is based on the fact that an $\epsilon$-net $\{x_1. \ldots, x_n\}$ is already given. In general, determining such elements of $X$ might be challenging, especially when considering a data space which only approximately represents the actual domain of $T$. When dealing with function spaces, as it is usually the case in real applications, we can approximate the elements $x_i$ directly as part of the learning algorithm. 
		
		In the present article, our algorithm learns the elements $x_i$ used to create the $\epsilon$-net for the Leray-Schauder projection by means of ``basis'' neural networks. It is therefore natural to ask whether the resulting model, where we learn both the operator in the space $E_n$ and the elements whose span give $E_n$ is still a universal approximator. The next result, which adapts the arguments of \cite{EZ_proj} to the case of the algorithm in this article, shows that the previous question has a positive answer. 
		
		In this article we will consider neural networks with continuous nonlinearities, so that the resulting neural networks are automatically continuous. 
		
		\begin{theorem}\label{thm:LS_NN}
					Let $X$ and $Y$ be the Banach spaces of continuous functions over compacts $K_1\subset \mathbb R^{N_1}$ and $K_2\subset \mathbb R^{N_2}$ with the uniform norm, $X = C(K_1,\mathbb R^{M_1})$ and $Y = C(K_2,\mathbb R^{M_2})$. Let $T: X \longrightarrow Y$ be a continuous operator, and let $\mathbb K\subset X$ be a compact subset. Then, for a given arbitrary $\epsilon>0$, it is possible to find neural networks $\{g_i : \mathbb R^{N_1} \longrightarrow \mathbb R^{M_1}\}_{i=1}^n$, $\{h_i : \mathbb R^{N_2} \longrightarrow \mathbb R^{M_2}\}_{i=1}^m$, and $f_{n,m}$ such that for every $x\in \mathbb K$
					\begin{eqnarray}\label{eqn:LS_NN}
					||T(x) - \phi_m^{-1}f_{n,m}\phi_n\tilde P_n(x)||_\infty < \epsilon,
					\end{eqnarray}
					where $\phi_k$ indicates an isomorphism between the finite dimensional space $F_k$ and the Euclidean space, and $\tilde P_n$ indicates the Leray-Schauder mapping on the space spanned by the neural networks $g_i$, and the $h_i$ neural networks span the space $E_m$. 
		\end{theorem}
		\begin{proof}
					We fix $\epsilon > 0$. As in \cite{EZ_proj}, we can split the proof in two steps, one where $T$ is uniformly continuous on a neighborhood $U$ of $\mathbb K$, and a general step where the initial map is approximated by a map uniformly continuous on a neighborhood $U$ of $\mathbb K$. Since the latter step is independent of the current setup, we can assume that $T$ is uniformly continuous on $U \supset \mathbb K$.
					It is possible to find functions $x_i$, with $i=1,\ldots, n'$, such that 
					\begin{eqnarray}\label{eqn:Pn_approx}
							\|x-P_n(x)\|_\infty &<& \frac{\delta}{3}, 
					\end{eqnarray}
					\begin{eqnarray}\label{eqn:T_approx}
					\|T(z_1)-T(z_2)\|_\infty &<& \frac{\epsilon}{3}, 
					\end{eqnarray}
					for all $x\in \mathbb K$, and for all $z_1,z_2\in U$ with $\|z_1-z_2\|_\infty < \delta$. 
					Moreover, $\delta$ is chosen such that $B(x,\delta) \subset U$ for each $x\in \mathbb K$.
					Here the continuous maps $P_n$ map $\mathbb K$ in the space $E_n$ spanned by $x_i$, $i=1, \ldots, n'$, where $\dim E_n = n$. In what follows we assume that $n' = n$ for simplicity of notation, as the proof would not change if otherwise.  
					
					The fundamental step of the present proof is to show that we can find neural networks $g_i$ spanning a space $F_n$ which has the property that $\|x - \tilde P_nx\|_\infty < \delta$ for all $x\in \mathbb K$, where $\tilde P_n$ is a mapping on $F_n$ which we define below. This would allow us to replace the approximation in \eqref{eqn:Pn_approx}, by an analogous approximation using mappings $\tilde P$ and a neural network basis instead of the functions $x_i$.  
					
					To this purpose, we  use the traditional universal approximation results as in \cite{Horn,Fun,LLPS,Lu,Pink} to find neural networks $g_i$ such that $\|x_i-g_i\|_\infty$ is arbitrarily small for each $i=1,\ldots, n$, where $x_i$ were found above. 
					We define
					\begin{eqnarray}
					\tilde P_n x = \frac{\sum_{i=1}^n \tilde\mu_i(x) g_i}{\sum_{j=1}^n\tilde\mu_j(x)},
					\end{eqnarray}
					where
					\begin{eqnarray*}
							\tilde\mu_i(x) &=& \begin{cases}
									\frac{\delta}{3} - \|x-g_i\|_\infty,  \hspace{1cm} \|x-g_i\|_\infty \leq \frac{\delta}{3}\\
									0,  \hspace{3.15cm} \|x-g_i\|_\infty > \frac{\delta}{3}.
							\end{cases} 
					\end{eqnarray*} 
					In addition, upon taking $g_i$ such that $\|x_i-g_i\|_\infty$ is sufficiently small, we can make $|\mu_i(x) - \tilde \mu_i(x)|$ arbitrarily small as well. 
						We have that 
					\begin{eqnarray}\label{eqn;min_sum_mui_uniform}
					\min_{x\in\mathbb K}\sum_{j=1}^n\mu_j(x) > 0, 
					\end{eqnarray} 
					as a consequence of the continuity of the map $\sum_{j=1}^n \mu_i: \mathbb K \longrightarrow \mathbb R^+_0$, the compactness of $\mathbb K$, and the fact that for each $x\in \mathbb K$ there exists $i$ such that $\|x - x_i\|_\infty < \frac{\delta}{3}$, and therefore for which $\mu_i(x) > 0$. We put $m_0 :=	\min_{x\in\mathbb K}\sum_{j=1}^n\mu_j(x)$. 
					
					Suppose $|\mu_i(x) - \tilde \mu_i(x)|< q$, where $q$ is small enough such that $\sum_{i=1}^n \mu_i(x) - q > 0$. This can be done since $m>0$, and it guarantees that $\sum_{j=1}^n\tilde\mu_j(x)$ does not vanish. Then, for all $i=1,\ldots, n$, we have that 
					\begin{eqnarray*}
							\lefteqn{|\frac{\mu_i(x)}{\sum_{j=1}^n\mu_j(x)} - \frac{\tilde\mu_i(x)}{\sum_{j=1}^n\tilde\mu_j(x)}|}\\
							&\leq& \sum_{j=1}^n\frac{1}{|(\sum_{j=1}^n\mu_j(x))(\sum_{j=1}^n\tilde\mu_j(x))|}|\mu_i(x)\tilde \mu_j(x) - \mu_j(x)\tilde \mu_j(x)|\\
							&\leq& \sum_{j=1}^n \frac{1}{m_0(m_0-q)}[|\mu_i(x)\tilde\mu_j(x) - \mu_i(x)\mu_j(x)| + |\mu_i(x)\mu_j(x)-\tilde \mu_i(x)\mu_j(x)|]\\
							&=& \sum_{j=1}^n \frac{1}{m_0(m_0-q)}[|\mu_j(x)|\cdot |\tilde \mu_j(x)-\mu_j(x)|+|\mu_j(x)|\cdot |\mu_i(x)-\tilde \mu_i(x)|]\\
							&<& \frac{2n\delta q}{3m_0(m_0-q)},
					\end{eqnarray*}
					which implies that $|\frac{\mu_i(x)}{\sum_{j=1}^n\mu_j(x)} - \frac{\tilde\mu_i(x)}{\sum_{j=1}^n\tilde\mu_j(x)}| \longrightarrow 0$ as $q\longrightarrow 0$. 
					We therefore find that by taking $g_i$ such that $\|x_i-g_i\|_\infty$ is small enough, we can obtain
					\begin{eqnarray*}
						\sum_{i=1}^n |\frac{\mu_i(x)}{\sum_{j=1}^n\mu_j(x)} - \frac{\tilde\mu_i(x)}{\sum_{j=1}^n\tilde\mu_j(x)}|
						&<& \frac{\delta}{3\max_i\|g_i\|_\infty}, 
					\end{eqnarray*} 
					and at the same time have $\|x_i-g_i\|_\infty < \frac{\delta}{3}$. 
					
					Then, in this situation we have
					\begin{eqnarray*}
							\|x - \tilde P_nx\|_\infty &\leq& \|x - P_nx\|_\infty + \|P_nx - \tilde P_nx\|_\infty \\
							&=& \frac{\delta}{3} + \|\frac{\sum_i \mu_i(x)x_i}{\sum_j\mu_j(x)} - \frac{\sum_i \tilde\mu_i(x)g_i}{\sum_j\tilde\mu_j(x)}\|_\infty\\
							&\leq& \frac{\delta}{3} + \sum_{i=1}^n \|\frac{\mu_i(x)x_i}{\sum_{j=1}^n\mu_j(x)} - \frac{\tilde\mu_i(x)g_i}{\sum_{j=1}^n\tilde\mu_j(x)}\|_\infty\\
							&\leq&  \frac{\delta}{3} + \sum_{i=1}^n[\|\frac{\mu_i(x)x_i}{\sum_{j=1}^n\mu_j(x)}-\frac{\mu_i(x)g_i}{\sum_{j=1}^n\mu_j(x)}\|_\infty+\|\frac{\mu_i(x)g_i}{\sum_{j=1}^n\mu_j(x)}-\frac{\tilde\mu_i(x)g_i}{\sum_{j=1}^n\tilde\mu_j(x)}\|_\infty]\\
							&\leq& \frac{\delta}{3} + \sum_{i=1}^n[\frac{\mu_i(x)}{\sum_{j=1}^n\mu_j(x)}\|x_i-g_i\|_\infty + \|g_i\|_\infty |\frac{\mu_i(x)}{\sum_{j=1}^n\mu_j(x)} - \frac{\tilde\mu_i(x)}{\sum_{j=1}^n\tilde\mu_j(x)}|]\\
							&<& \frac{\delta}{3} + \frac{\delta}{3} + \frac{\delta}{3}\\
							&=& \delta,
					\end{eqnarray*}
					where we have used the fact that $\mu_j(x) \geq 0$ by definition of Leray-Schauder mapping. A similar procedure can be used to find mappings $\tilde P_m$ for the compact $T(\mathbb K)$, where the centers of the $\frac{\epsilon}{3}$-balls in \cite{EZ_proj} are taken to be neural networks.
					
					We let $\phi_n$ be an isomorphism which identifies $F_n$ with $\mathbb R^{\dim F_n}$, and construct the mapping $T_{n,m}$ induced by $T$ on $F_n$ as in Theorem~2.1 of \cite{EZ_proj}.  We find a neural network architecture which approximates $T_{n,m}$ on the compact $\phi_n\tilde P_n(\mathbb K)$. With these definitions one can verify that the approximation \eqref{eqn:LS_NN} holds for all $x\in \mathbb K$, by considering the same chain of inequalities in Theorem~2.1 of \cite{EZ_proj} with $\tilde P_n$ replacing $P_n$. 
		\end{proof}
		
		\begin{remark}
			{\rm 
					We note that the proof of Theorem~\ref{thm:LS_NN} is showing an extra bit of information, other than the approximation of $T$. It is showing that this can be done by also approximating the analytically defined maps $P_n$ by neural network counterparts $\tilde P_n$. 
			}
		\end{remark}
			
		Theorem~\ref{thm:LS_NN} produces an approximation result similar to that of \cite{chen}, and consequently that of DeepONet \cite{DeepOnet}. However, the present approach is based on significantly different methods. In particular, as a byproduct of this approach, the nonlinear maps used to ``project'' on a finite span of neural network basis functions are obtained by approximating Leray-Schauder maps whose analytical formulation is known a priori. These are the maps $\tilde P_n$ used in the proof above. The functionals used in DeepONet to nonlinearly ``project'' on the finite dimensional space of trunk networks, do not approximate analytically defined maps, and it is therefore not clear a priori what they need to approximate for the universal approximation results to hold. 
		This fact has also implications on the algorithmic implementation given below, as the maps $\tilde P_n$ are implemented with neural networks explicitly.  
		
		The construction of $\tilde P_n$, in practice, implies the fact that we need to map on a space spanned by neural networks through the use of the functions $\tilde\mu_i$, where we need to find $g_i$ such that the maps $\tilde\mu_i$ have the required properties. We would like to implement $\tilde \mu_i$ in $\tilde P_n$ as neural networks as well, 
		so that the whole architecture consists of neural network components only.  
		In the rest of this section we will work toward this direction, showing that it is indeed possible to replace $\tilde P_n$ by neural networks. This construction will be seen to be a universal approximation architecture on functions that are regular enough with respect to the norms of their derivatives. We will denote by $L^p_\mu(K)$ the $L^p$-space on a compact $K$ for the measure $\mu$ which is assumed to be Lebesgue and normalized on $K$ for simplicity, although the reasoning can be adapted to finite Borel measures. We consider $1\leq p < +\infty$, and we denote the H\"older space  by $C^{\kappa,\alpha}([a,b])$, where $\kappa\geq 1$, and $\alpha\in (0,1]$. We assume that $p\geq \kappa + \alpha$.
		We will later mention how to generalize the result to the multivariate case. This relies on cubature formulas, and our choice of proceeding with the single varaible case first, is motivated by the wish of keeping the notation as simple as possible.  
		
		For an interval $[a,b]$ of $\mathbb R$, for a discretization (e.g. a partition) $\{t_1,\ldots, t_k\} \subset [a,b]$ and a neural network $F: \mathbb R^k \longrightarrow \mathbb R^s$, we define $F(f) := F(f(t_1),\ldots, f(t_k))$ for any function $f: [a,b] \longrightarrow \mathbb R$.
		In other words, $F$ is applied to the evaluation of $f$ on the nodes of discretization. In machine learning, one usually has access to data sampled on discretizations of the domains of definition. Therefore, the neural network can work as a proxy of an operator. We leverage this fact to approximate the Leray-Schauder mappings $P_n$ via neural networks. In what follows we restrict ourselves to quadrature rules that are continuous on their inputs. We will say that a function $u$ is sufficiently smooth if it is such that $|u|^p$ is in $C^{\kappa,\alpha}([a,b])$, e.g. it is $k+2$ times differentiable with all continuous derivatives in a neighborhood of $[a,b]$ since $p\geq\kappa+\alpha$. 

		\begin{lemma}\label{lem:Pn_NN}
					Let $\mathbb K\subset L^p_\mu([a,b])$ be a compact subspace, and let $P_n$ be the Leray-Schauder map $P_n : \mathbb K \longrightarrow E_n$, where $E_n$ is spanned by an $\epsilon$-net of sufficiently smooth functions $\{x_1, \ldots, x_n\}$. Then, for any $\eta>0$, there exists a discretization $\{t_1,\ldots, t_k\}$ of $[a,b]$, and neural networks $F^i: \mathbb R^k \longrightarrow \mathbb R$ such that
					\begin{eqnarray*}
							\|P_n(x) - \sum_{i=1}^n\frac{F^i(x)x_i}{\sum_{j=1}^nF^j(x)}\|_p &<& \eta,
					\end{eqnarray*}
					for all $x\in \mathbb K\bigcap \bar B_\rho$, where $B_\rho$ is an open ball in the H\"older sapce $C^{\kappa+1,\alpha}([a,b])$, 
					for any $\rho>0$ such that $\max_i\|x_i\|_\infty \leq \rho$. 
		\end{lemma}
		\begin{proof}
				As in the case of Theorem~\ref{thm:LS_NN}, letting $\mu_i$ indicate the maps used to define $P_n$, we have 
				\begin{eqnarray}\label{eqn;min_sum_mui}
						\min_{x\in\mathbb K}\sum_{j=1}^n\mu_j(x) > 0.
				\end{eqnarray} 
				We put $m : =	\min_{x\in\mathbb K}\sum_{j=1}^n\mu_j(x)$. We also set $M := \max_{x\in\mathbb K}\sum_{j=1}^n\mu_j(x)$, which is a finite number since each $\mu_i$ is continuous over a compact.
				
				We now introduce an integration scheme such that we can approximate the integrals that give $\|\cdot\|_p$ by a quadrature rule uniformly on elements of $\mathbb K$ that also lie in a given H\"older ball centered at the origin, $\bar B_\rho$, as in the statement of the lemma. For simplicity, we can think of this as being the forward rectangle rule, even though the same argument can be applied to other rules that converge faster (e.g. Trapezoidal, Cavalieri-Simpson, Boole) upon changing the error bounds, and taking $k$ large enough in the statement of this lemma, depending on the value of $\kappa$. A similar procedure was used in \cite{uni} to approximate integral operators with applications to the transformer architecture. For $u\in\bar B_\rho$, we have that $\sup_{[a,b]} |u'| \leq \rho$, and therefore we find that the quadrature rule gives uniform error bounds of $e_k = \frac{b-a}{2k} \rho$ for all $u\in \mathbb K\bigcap \bar B_\rho$, where $k$ is the number of points used to discretize the interval $[a,b]$. We indicate the quadrature rule to approximate the integral $\int_a^bfd\mu$ by $\frak I(f)$. In addition, as $k\longrightarrow +\infty$ we have that $|\hat \mu_i(x) - \mu_i(x)| \longrightarrow 0$ for all $i=1,\ldots, n$ and for all $x\in \mathbb K\bigcap \bar B_\rho$, where we define
				\begin{eqnarray}\label{eqn:hat_mui}
				\hat\mu_i(x) = \begin{cases}
				\epsilon - \frak I(|x-x_i|^p)^{\frac{1}{p}}, \hspace{1cm} \frak I(|x-x_i|^p)^{\frac{1}{p}} \leq \epsilon\\
				0, \hspace{3.6cm} \frak I(|x-x_i|^p)^{\frac{1}{p}} > \epsilon
				\end{cases}
				\end{eqnarray}
				for all $i=1, \ldots , n$. 
				We define $\hat P_n(x) := \sum_{i=1}^n\frac{\hat\mu_i(x)x_i}{\sum_{j=1}^n\hat\mu_j(x)}$. Then, for all $x\in \mathbb K\bigcap \bar B_\rho$, whenever $k$ is large enough so that $m-e_k>0$, we have
				\begin{eqnarray*}
						\|\hat P_n(x) - P_n(x)\|_p
						&\leq& \sum_{i=1}^n \|\frac{\hat\mu_i(x)x_i}{\sum_{j=1}^n\hat\mu_j(x)} - \frac{\mu_i(x)x_i}{\sum_{j=1}^n\mu_j(x)}\|_p\\
						&=& \sum_{i=1}^n \|x_i\|_p\cdot |\frac{\hat\mu_i(x)}{\sum_{j=1}^n\hat\mu_j(x)} - \frac{\mu_i(x)}{\sum_{j=1}^n\mu_j(x)}|\\
						&\leq& \sum_{i=1}^n \frac{ \|x_i\|_p}{(m-e_k)m}\cdot |\sum_{j=1}^n\hat\mu_i(x)\mu_j(x) - \sum_{j=1}^n\mu_i(x)\hat\mu_j(x)|\\
						&\leq& \sum_{i=1}^n \frac{ \|x_i\|_p}{(m-e_k)m}\cdot
									\sum_{j=1}^n |\hat\mu_i(x)\mu_j(x) - \mu_i(x)\hat\mu_j(x)|\\
						&\leq& \sum_{i=1}^n \frac{ \|x_i\|_p}{(m-e_k)m}\cdot
						\sum_{j=1}^n [|\hat\mu_i(x)\mu_j(x) - \mu_i(x)\mu_j(x)|+|\mu_i(x)\mu_j(x) - \mu_i(x)\hat\mu_j(x)|]\\
						&=&  \sum_{i=1}^n \frac{ \|x_i\|_p}{(m-e_k)m}\cdot
						\sum_{j=1}^n [|\mu_j(x)|\cdot|\hat\mu_i(x) - \mu_i(x)|+|\mu_i(x)|\cdot|\mu_j(x) - \hat\mu_j(x)|]\\
						&\leq& \sum_{i=1}^n \frac{ \|x_i\|_p}{(m-e_k)m}\cdot\sum_{j=1}^n2Me_k\\
						&=& \frac{2n^2Me_k}{m^2-me_k}\cdot \max_i \|x_i\|_p.
				\end{eqnarray*}
				As we increase the discretization points $k$, i.e. as $k\longrightarrow +\infty$, it follows that $\|\hat P_n(x) - P_n(x)\|_p \longrightarrow 0$, and we can therefore approximate $P_n$ by means of the map $\hat P_n$ induced by the quadrature rule $\frak I$ over $\mathbb K\bigcap \bar B_\rho$, as long as we choose $k$ large enough. The last step of the proof, is to approximate the quadrature rule by means of a neural network. 
				
				Since the functions $\hat \mu_i : \mathbb R^k \longrightarrow \mathbb R^+_0$ are continuous with respect to the inputs of the discretization of the integration scheme $\frak I$, we can approximate each function $\hat \mu_i$ by a neural network $F^i_\theta$, according to the results in \cite{Horn,Fun,LLPS,Lu,Pink} with arbitrarily high precision for each $x\in \mathbb K\bigcap \bar B_\rho$. In particular, we can find architectures such that 
				\begin{eqnarray*}
						|\frac{\hat\mu_i(x)}{\sum_{j=1}^n \hat \mu_j(x)} - \frac{F^i_\theta(x)}{\sum_{j=1}^n F^j_\theta(x)}| < \frac{\eta}{2n\max_i \|x_i\|_p}, 
				\end{eqnarray*}
				for all $x\in \mathbb K\bigcap \bar B_\rho$. Then, if $k$ is large enough such that $\frac{2n^2Me_k}{m^2-me_k}\cdot \max_i \|x_i\|_p < \frac{\eta}{2}$, we find that for all $x\in \mathbb K\bigcap \bar B_\rho$ we have
				\begin{eqnarray*}
						\|P_n(x) - \sum_{i=1}^n\frac{F^i(x)x_i}{\sum_{j=1}^nF^j(x)}\|_p
						&\leq& \|P_n(x) - \hat P_n(x)\|_p + \|\hat P_n(x) -  \sum_{i=1}^n\frac{F^i(x)x_i}{\sum_{j=1}^nF^j(x)}\|_p\\
						&\leq& \frac{2n^2Me_k}{m^2-me_k}\cdot \max_i \|x_i\|_p + \sum_{i=1}^n\|\frac{\hat \mu_i(x)x_i}{\sum_{j=1}^n \hat \mu_j(x)} - \frac{F^i_\theta(x)x_i}{\sum_{j=1}^nF^j_\theta(x)}\|_p\\
						&<&  \frac{\eta}{2} + \sum_{i=1}^n\|x_i\|_p\cdot |\frac{\hat \mu_i(x)}{\sum_{j=1}^n \hat \mu_j(x)} - \frac{F^i_\theta(x)}{\sum_{j=1}^nF^j_\theta(x)}|\\
						&<& \frac{\eta}{2} +  \frac{\eta}{2}\\
						&=& \eta. 
				\end{eqnarray*}
				This completes the proof. 
		\end{proof}
	
		\begin{remark}
			{\rm 
					We notice that upon taking some suitable nonlinearity, we can also ensure that the output of each $F^i_\theta$ in the proof of Lemma~\ref{lem:Pn_NN} is nonnegative. In this sense, the map $x \mapsto (\frac{F^1_\theta(x)}{\sum_{j=1}^n F^j_\theta(x)}, \ldots, \frac{F^n_\theta(x)}{\sum_{j=1}^n F^j_\theta(x)})^T$ resembles the components of a Leray-Schauder map $P_n$. 
			}
		\end{remark}
	
		\begin{proposition}
					In Lemma~\ref{lem:Pn_NN} we can choose the neural network $F$ to consist of a single hidden layer feed forward architecture. 
		\end{proposition}
		\begin{proof}
					For $x\in \mathbb K \bigcap \bar B_\rho$ we have that $\|x\|_\infty \leq \rho$. Therefore, for a discretization $\{t_1,\ldots, t_k\}$ we have that $\max_{\ell,i} |x(t_\ell) - x_i(t_\ell)| \leq 2\rho$. It follows that the inputs of $\frak I$ are in the bounded hypercube $[-q,q]^{\times k}$, where $q = 2^p\rho^p$. We can therefore apply the results for universal approximation over compacts.
		\end{proof}
	
			In the following result, we consider nonlinearities for the neural networks that are sufficiently smooth, in the sense defined above.  
		
		\begin{theorem}\label{thm:P_NN}
					Let $T: X \longrightarrow X$ be a continuous map, where $X = L^p_\mu([a,b])$. Let $\mathbb K \subset X$ be a compact in $X$. Then, for any given $\epsilon > 0$ we can find $n, m\in \mathbb N$, $\rho>0$, and neural networks $\{g_i\}_{i=1}^n, \{h_j\}_{j=1}^m, f_{n,m}$ and $\{F^i\}_{i=1}^n$ such that 
					\begin{eqnarray}
							\|T(x) - \phi_m^{-1}f_{n,m}\phi_n(\sum_{i=1}^n\frac{F^i(x)g_i}{\sum_{j=1}^nF^j(x)})\|_p &<& \epsilon, 
					\end{eqnarray}
					for all $x\in \mathbb K \bigcap \bar B_\rho$, where $\bar B_\rho$ is a radius $\rho$ closed ball in $C^{\kappa+1,\alpha}([a,b])$, and the neural networks $h_i$ span the domain of $\phi_m$.
		\end{theorem}
		\begin{proof}
					The proof is similar to that of Theorem~\ref{thm:LS_NN}, where the details following Theorem~\ref{thm:Universal} do not depend on the use of the uniform norm. The main difference here lies in the use of Lemma~\ref{lem:Pn_NN}, to construct a neural network Leray-Schauder mapping $F$ which plays the role of $P_n$ in Theorem~\ref{thm:Universal}, and $\tilde P_n$ in Theorem~\ref{thm:LS_NN}. We therefore provide only the details pertaining to this aspect of the proof. 
					In the construction of $P_n$, we find continuous functions $\{x_1,\ldots, x_n\}$  such that  $\mathbb K$ is contained in $\delta$-balls centered at $x_i$. Furthermore, since we can approximate each $x_i$ by neural networks $g_i$ in the uniform norm, we can assume that the each element used to construct $P_n$ is a neural network, and upon taking sufficiently smooth nonlinearities, we can assume that they are in $C^{\kappa+1,\alpha}([a,b])$, so that the function $|g_i|^p$ is in $C^{\kappa,\alpha}([a,b])$, following essentially the same procedure as in Theorem~\ref{thm:LS_NN}.  
					We assume that the $g_i$'s span an $n$-dimensional space to simplify notation, since the proof is adapted directly as in \cite{EZ_proj} if otherwise.   
					Then, applying Lemma~\ref{lem:Pn_NN} we can find neural networks $F^i$ such that the Leray-Schauder mapping $P_n$ can be approximated by the mapping $x\mapsto \sum_{i=1}^n\frac{F^i(x)g_i}{\sum_{j=1}^nF^j(x)}$ for all $x\in \mathbb K\bigcap \bar B_\rho$ with arbitrary precision, where $\rho \geq \max_i \|g_i\|_\infty$. Now, to apply the arguments of Theorem~\ref{thm:Universal} and Theorem~\ref{thm:LS_NN} there is one last subtlety to handle. While $P_n(\mathbb K)$ is a compact in $E_n$, it is not a priori clear whether the image of $x\mapsto \sum_{i=1}^n\frac{F^i(x)g_i}{\sum_{j=1}^nF^j(x)}$, for $x\in \mathbb K \bigcap \bar B_\rho$ is compact. However, by construction we have that 
					\begin{eqnarray*}
							\|\sum_{i=1}^n\frac{F^i(x)g_i}{\sum_{j=1}^nF^j(x)}\|_p 
							&\leq& \sum_{i=1}^n\|\frac{F^i(x)g_i}{\sum_{j=1}^nF^j(x)}\|_p\\
							&=&  \sum_{i=1}^n|\frac{F^i(x)}{\sum_{j=1}^nF^j(x)}|\cdot \|g_i\|_p\\
							&\leq&  \sum_{i=1}^n|\frac{\delta+\eta}{m_0-\eta}|\cdot  \|g_i\|_p\\
							&\leq& n\frac{\delta+\eta}{m_0-\eta}\max_i \|g_i\|_p,
					\end{eqnarray*}
					where we have denoted by $\eta$ the error in the approximation $P_n$ via  the mapping $x\mapsto \sum_{i=1}^n\frac{F^i(x)g_i}{\sum_{j=1}^nF^j(x)}$, we used the fact that $\mu_i(x) \leq \delta$ always, we have denoted by $m_0$ the minimum of $P_n$ on $\mathbb K$, and we used the properties obtained in the proof of Lemma~\ref{lem:Pn_NN}. Then, the image of the mapping $x\mapsto \sum_{i=1}^n\frac{F^i(x)g_i}{\sum_{j=1}^nF^j(x)}$ is bounded in the finite dimensional space $E_n$ for all $\eta$ sufficiently small. We can therefore bound it by a closed ball in $E_n$ which is compact, since $E_n$ with the norm induced by the $p$ norm of $X$ is equivalent to the $n$-dimensional Euclidean space.  
					So, the neural network $f_{n,m}$ takes inputs for the values of $P_n(\mathbb K)$ and the image of $x\mapsto \sum_{i=1}^n\frac{F^i(x)g_i}{\sum_{j=1}^nF^j(x)}$ over a compact, where it is uniformly continuous. Therefore, we can choose the $\eta$ approximation of $P_n$ by $x\mapsto \sum_{i=1}^n\frac{F^i(x)g_i}{\sum_{j=1}^nF^j(x)}$ such that $\|f_{n,m}(\phi_nP_n(x)) - f_{n,m}\phi_n\sum_{i=1}^n\frac{F^i(x)g_i}{\sum_{j=1}^nF^j(x)}\|_p < \frac{\epsilon}{\|\phi_m^{-1}\|}$. Then, we have
					\begin{eqnarray*}
							\|T(x) - (\phi_m^{-1}f_n\phi_n)\sum_{i=1}^n\frac{F^i(x)g_i}{\sum_{j=1}^nF^j(x)}\|_p
							&\leq& \|T(x) - TP_n(x)\|_p + \|TP_n(x) - P_mTP_n(x)\|_p\\
							&& + \|P_mTP_n(x) - (\phi_m^{-1}f_{n,m}\phi_n)P_n(x)\|_p\\
							&& + \|(\phi_m^{-1}f_{n,m}\phi_n)P_n(x) - (\phi_m^{-1}f_{n,m}\phi_n)\sum_{i=1}^n\frac{F^i(x)g_i}{\sum_{j=1}^nF^j(x)}\|_p\\
							&<& \frac{\epsilon}{3} + \frac{\epsilon}{3} + \frac{\epsilon}{3} + \epsilon\\
							&=& 2\epsilon. 
					\end{eqnarray*}
					Since $\epsilon$ was arbitrary, the proof is complete.
		\end{proof}
	
		\begin{remark}
			{\rm 
				We notice that the proof can proceed in the same way if $T: X \longrightarrow Y$ where $Y = L^q_\mu([a,b])$ with $q$ possibly different from $p$.
			}
		\end{remark}

		We now discuss the generalization of the previous result to the multivariate case $L^p([a,b]^r)$, in which we will replace the quadrature rules in Lemma~\ref{lem:Pn_NN} by cubature formulas \cite{Sobolev}, restricting ourselves to rules that are continuous in their inputs. We consider the space $C^{\kappa+1,\alpha}([a,b]^r)$, on $[a,b]^r$, and assume as before $p$ sufficiently large. As in the previous result, the approximation holds for those functions in $\mathbb K$ with some extra regularity. 
		
		\begin{theorem}\label{thm:P_NN_higherD}
				Let $T: X \longrightarrow Y$ be continuous map, where $X = L^{p_1}_\mu([a,b]^{r_1})$ and $Y = L^{p_2}_\mu([c,d]^{r_2})$. Let $\mathbb K \subset X$ be a compact in $X$. Then, for any given $\epsilon > 0$ we can find $n, m\in \mathbb N$, $\rho>0$, and neural networks $\{g_i\}_{i=1}^n, \{h_j\}_{j=1}^m, f_{n,m}$ and $\{F^i\}_{i=1}^n$ such that 
				\begin{eqnarray}
				\|T(x) - \phi_m^{-1}f_{n,m}\phi_n(\sum_{i=1}^n\frac{F^i(x)g_i}{\sum_{j=1}^nF^j(x)})\|_{p_2} &<& \epsilon, 
				\end{eqnarray}
				for all $x\in \mathbb K \bigcap \bar B_\rho$, where $B_\rho$ is a radius $\rho$ ball in $C^{\kappa+1,\alpha}([a,b]^r)$ with $p_1\geq \kappa+ \alpha$, and the neural networks $h_i$ span the domain of $\phi_m$.
		\end{theorem}
		\begin{proof}
				The proof proceeds in the same way as for Theorem~\ref{thm:P_NN}, with the main difference lying in Lemma~\ref{lem:Pn_NN}. In this case we use cubature rules to construct maps $\hat \mu_i$ to approximate the operators $P_n$. Restricting to those elements of $\mathbb K$ that are also in $\bar  B_\rho$ guarantees that by increasing the grid points used for the cubature formulas, the approximation $\hat\mu_i$ can be made arbitrarily precise. Moreover, the inputs of the cubatures lie in some hypercube of the Euclidean space, so that the neural networks can be chosen to approximate the cubaature formulas on a compact. 
		\end{proof}
	
		\section{Algorithm}\label{sec:Algo} 
		
		 In this section, we describe more in detail the concrete resulting algorithm from the previous theoretical results. Our deep learning algorithm is based on constructing a function input from the data available for initialization (e.g. an initial condition or boundary values, depending on the formulation of the problem), and (nonlinearly) project the input on a basis of functions modeled with neural networks. The projection maps are Leray-Schauder maps, defined through Equation~\eqref{eqn:LS}. The maps $\mu_i$ are either fixed as in Equation~\eqref{eqn:mui} (for some choice of norm), or are learned as in Theorem~\ref{thm:LS_NN} and Theorem~\ref{thm:P_NN}. In fact, while our experiments show that the assignment in Equation~\eqref{eqn:mui} gives a functioning model, we have found that hyperparameter fine tuning is more complex in this setting, and simply learning the $\mu_i$ maps, which in turn means that we learn the Leray-Schauder maps as well, makes training much simpler. 
		 
		Following the results in Section~\ref{sec:Leray_Schauder}, the model consists of two basis sets of neural networks, and a neural network $f_{n,m}$ that operates on the spaces spanned by the bases. We point out some strong similarities with the DeepONet model, where one has a basis of neural networks (the trunk networks), and a set of functionals that project on such basis. In the present case, one can think of the first set of Leray-Schauder approximations and $f_{n,m}$ as playing the role of the projection functionals in DeepONet, and the second set of Leray-Schauder approximations as playing the role of the trunk neural networks.  However, the resulting model does not seem to be equivalent to DeepONet, as it is not a priori clear whether the functionals of DeepONets can always decomposed into a Leray-Schauder approximation plus a feedforward neural network $f_{n,m}$. 
		
		Since we also learn the functions $\mu_i$ in Equation~\eqref{eqn:mui}, in practice we consider convolutional neural networks $\mu_i: C(K,\mathbb R^M) \longrightarrow \mathbb R^+_0$. We use CNNs to implement the $\mu_i$ neural networks because they mimic an integration procedure which can be performed in higher dimensions. They take a discretized input function and produce a numerical value (which can be constrained to be nonnegative by choice of the final nonlinearity to ensure increased stability).  
		
				We summarize the construction by means of the following commutative diagram
				\begin{center}
						\begin{tikzcd}
							\mathbb K \arrow[rr,"T^\theta_{n,m}"]\arrow[d,"P_n"] &  & Y\\
							E_n \arrow[rr,"T_{n,m}",dashed] & & E_m\arrow[u,hook]\\
							\mathbb R^n\arrow[rr,"f_{n,m}"]\arrow[u,equal] & & \mathbb R^n\arrow[u,equal]
						\end{tikzcd}, 
				\end{center}
				where $T^\theta_{n,m}$ represents the introduced model, which we will call the {\it Leray-Schauder Neural Operator} for brevity. The map $T_{n,m}$ is obtained by considering the restriction of $P_mT$ on $E_n$, and it is drawn dashed because $f_{n,m}$ only approximates this map on $\mathbb K\cap \bar B_\rho$ as above, and therefore the diagram would not strictly commute with such $T_{n,m}$. We have indicated the identifications $\phi_n$ and $\phi_m$ by equality symbols.  
				The neural networks $\{h_j\}$ that span $E_m$ work in a fashion similar to the trunk neural networks of DeepONet, while the map $P_n$ is the approximation to the Leray-Schauder mapping $P_n$ of Section~\ref{sec:Leray_Schauder}. The composition $\hat \phi_m^{-1}f_{n,m}\phi_n P_n$ in this diagram would correspond to the functionals used in DeepONet. 
		
		 We use the initialization values (data available during inference) to create a function $\mathbf y(\mathbf x,t)$, where $\mathbf x\in C \subset \mathbb R^n$, and $t\in [a,b]$, where $K = C\times [a,b]$ in the setting of Theorem~\ref{thm:LS_NN}. This is the input to the operator. In practice, we obtain $\mathbf y$ by interpolating the input data, e.g. $\mathbf y(\mathbf x, 0)$ and $\mathbf y(\mathbf x, 1)$ in the experiments below. Then, we compute the Leray-Schauder mapping as
		 \begin{eqnarray*}
		 		P_n(\mathbf y) = \sum_{i=1}^n\frac{\mu^i(\mathbf y)g_i}{\sum_{j=1}^n\mu^j(\mathbf y)},
		 \end{eqnarray*}
		 which is the same operator used in the approximation of Theorem~\ref{thm:P_NN}, where $F^i$ has been named $\mu^i$ to mimic the definition in \eqref{eqn:mui}. Since $P_n(\mathbf y)$ can be evaluated through the networks $g_i$, it is clear that $P_n$ takes a function as input, obtained via initialization data, and produces a function as output. This function $P_n(\mathbf y)$ lives inside the span $\langle g_1, \ldots, g_n\rangle$. In fact, $P_n(\mathbf y)$ is a linear combination of $g_i$, with coefficients $q_i := \frac{\mu^i(\mathbf y)}{\sum_{j=1}^n\mu^j(\mathbf y)}$. The vector of coefficients $\mathbf q = (q_1,\ldots, q_n)^T\in \mathbb R^n$ is therefore the input of the network $f_{n,m}$, following the proof of Theorem~\ref{thm:P_NN}. 
		 
		We compute $\mathbf b = f(\mathbf q)$, and $\mathbf b$ is used to take a linear combination of the neural networks $h_i$, which we set $\psi = \sum_{i=1}^m b_ih_i$. We can then evaluate $\psi(\mathbf x,t)$ and compute the loss $\|\psi(\mathbf x,t) - \mathbf z(\mathbf x,t)\|$, where $\mathbf z$ is the target data which we are predicting. 
		
		This is summarized in Algorithm~\ref{algo:LS}. 
		
		\begin{algorithm}[h!]
			\caption{Algorithm for the Leray-Schauder neural operator.}
			\label{algo:LS}
			\begin{algorithmic}[1]
				\Require{Neural networks $\{g_i\}_{i=1}^n$, $\{h_j\}_{j=1}^m$ and $f_{n,m}$}
				\Ensure{Approximation $\phi_m^{-1}f_n\phi_n\tilde P_n$ of operator $T$}
				\State{Use initialization data to define input function $y$ (e.g. linear interpolation)}
				\State{Project $y$ on $\langle g_i\rangle$ using $\tilde P$: Coefficients $q_i$}
				\State{Use $\phi_n$ to identify coefficients $\mathbf q$ with elements of $\mathbb R^n$}
				\State{Apply $f_{n,m}$ on $\mathbf q$: $f_{n,m}(\mathbf q) = \mathbf b$}
				\State{Output function, using $\phi^{-1}_m$, is $\psi(\mathbf x,t) = \sum_{i=m}^n b_ih_i(\mathbf x,t)$}
				\State{Compute loss with target data $\mathbf z$: $\|\psi - \mathbf z\|$}
				\State{Compute gradients and use stochastic grandient descent to optimize}
			\end{algorithmic}
		\end{algorithm}

		\section{Experiments}\label{sec:exp}

		To demonstrate the capabilities of the method discussed in this article we experiment on two datasets. One consists of a set of spirals generated by solving an integral equation, and the other is a dataset on Burgers' equation. In both cases, the model has access at initialization to initial time $t=0$ and final time $t=1$ (initial and final time). The initialization is obtained by linearly interpolating between initial and final time configurations. 
		
			To test whether the model is stable under change of domain grid sampling, we consider two related situations. For the IE Spirals dataset, we consider the problem of interpolation. This means that the model is trained over a downsampled function, but tested on the full grid samples. It is seen that the model is stable under this upsampling procedure, due to the fact that the training procedure is effectively grid-independent. For the Burgers' dataset, we consider two different grid sampling sizes for training. We see that the model accuracy is substantially independent of this sampling specifications. It was seen in \cite{ANIE} that several models desplayed various behaviors under change of sampling size. While the ANIE model mitigated this via a Monte Carlo sampling procedure in the iterations of the solver, the present model is substantially more stable without the need of performing regularization procedures. In addition, the computational cost of the model is substantially unchanged when the grid size is varied. In fact, the main computational cost is related to the number of neural networks used for the two bases, and this value depends only on the problem at hand, and not on the grid size.  

		\begin{table}[h!]
			\caption{}\label{tab:spirals_Burger} 
			\centering
			\resizebox{\textwidth}{!}{\begin{tabular}{c c c | c c }
					\hline
					& \multicolumn{2}{c}{IE Spirals} & \multicolumn{2}{c}{Burgers'}\\
					\hline
					& Original & Interpolation & $s=256$ & $s=512$ \\
					\hline
				Leray-Schauder & $0.0011\pm0.0005$ &$0.0011\pm0.0005$ & $0.0017\pm 0.0012$ & $0.0017\pm0.0011$
					\\
				ANIE & $0.0014 \pm 0.0002$&$0.0015 \pm 0.0003$ & $0.0016\pm0.0016$ &	$0.0017\pm0.0018$
					\\
				Spectral NIE &$0.0022 \pm 0.0016$ & $0.0028 \pm 0.0032$ & -- &--
					\\
					FNO1D (init 5) &$0.0292 \pm 0.0285$ & $0.0994 \pm 0.1207$ & -- & --
					\\
					FNO1D (init 10) &$0.0286 \pm 0.0268$ & $0.1541 \pm 0.2724$ & -- & --
					\\
					\hline
			\end{tabular}}
		\end{table}\
	
		Both datasets were used in \cite{ANIE,Spectral}, and are publicly available at \url{https://figshare.com/articles/dataset/IE_spirals/25606242} and \url{https://figshare.com/articles/dataset/Burgers_1k_t400/25606149}. 
		
		The $2D$ integral equation spirals have been obtained by numerically solving an IE via iterative methods (Banach-Caccioppoli iterations). The solver is our implementation found in \cite{ANIE}, and whose code is publicly available. 
		The equation has the form:
		\begin{eqnarray*}
			\mathbf y(t) = \int_0^t \begin{bmatrix}
				\cos 2\pi(t-s) & -\sin 2\pi(t-s) \\
				-\sin 2\pi(t-s) & -\cos 2\pi(t-s)
			\end{bmatrix} \tanh(2\pi \mathbf y(s))dx
			+ \mathbf z_0 + \begin{bmatrix}
				\cos(t) \\
				\cos(t+\pi)
			\end{bmatrix},   
		\end{eqnarray*}
		where $z_0$ was sampled from a uniform distribution to derive different instances of the equation for the dataset.
		
		For the Burgers' equation, our dataset is generated using the Matlab code used in \cite{FNO}, which can be found in their GitHub page: 
		\url{https://github.com/zongyi-li/fourier\_neural\_operator/tree/master/data\_generation/burgers}. The solution is given on a spatial mesh of $1024$ and $400$ time points are generated from a random initial condition. We use $1000$ curves for training and test on $200$ unseen curves. Initial and final times are seen by the model at initialization, and all the time points constitute the ground truth for training the model. Some details regarding the Burgers' equation are given in \cite{ANIE}. 
		
		The implementational details of our model are found on our GitHub page  \url{https://github.com/emazap7/Leray_Schauder_neural_net}.

		In the experiments, we find that the model is comparable with the state of the art models (as in the experiments in \cite{ANIE}) on both datasets. For the IE spirals, we have included two experiments, one on prediction of the original dynamics, and another one on an interpolation task where the model has access to half the points during training and predicts all of them during evaluation. For the Burgers' dataset we have experimented on two spatial resolutions $s=256$ and $s=512$. An example of ground truth dynamics for the Burgers' equation is found in Figure~\ref{fig:obs}, while the corresponding model's prediction is in Figure~\ref{fig:pred}. The vertical direction indicates space and the horizontal direction indicates time.
		
		\begin{figure}[htb]
			\begin{center}
				\includegraphics[width=4in]{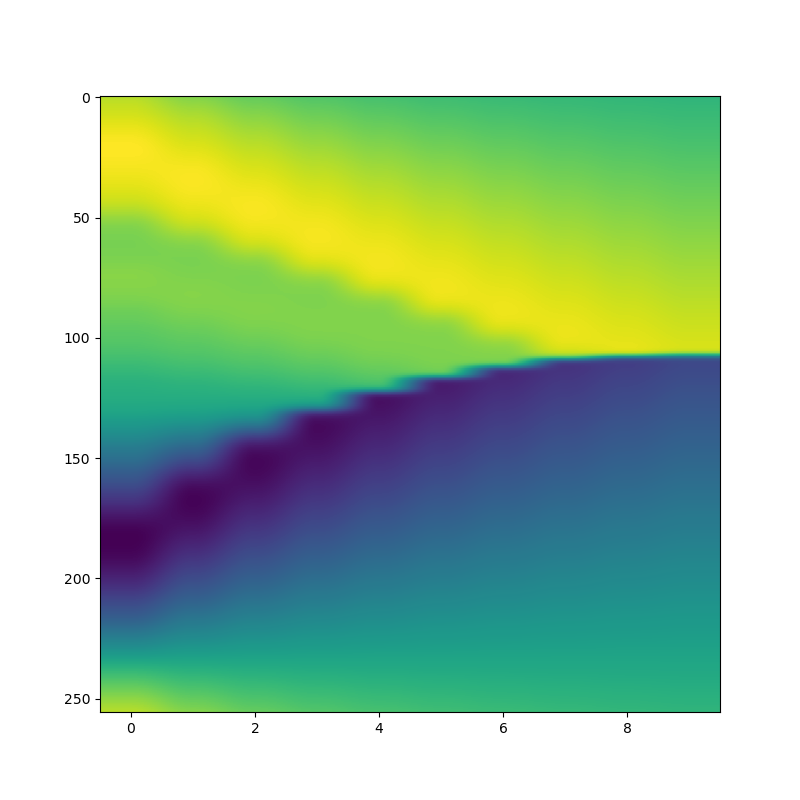}
			\end{center}
			\caption{Example of ground truth data for Burgers' dynamics where $\vec{x}$-axis is time and $\vec{y}$-axis represents space}
			\label{fig:obs}
		\end{figure}
	
		\begin{figure}[htb]
			\begin{center}
				\includegraphics[width=4in]{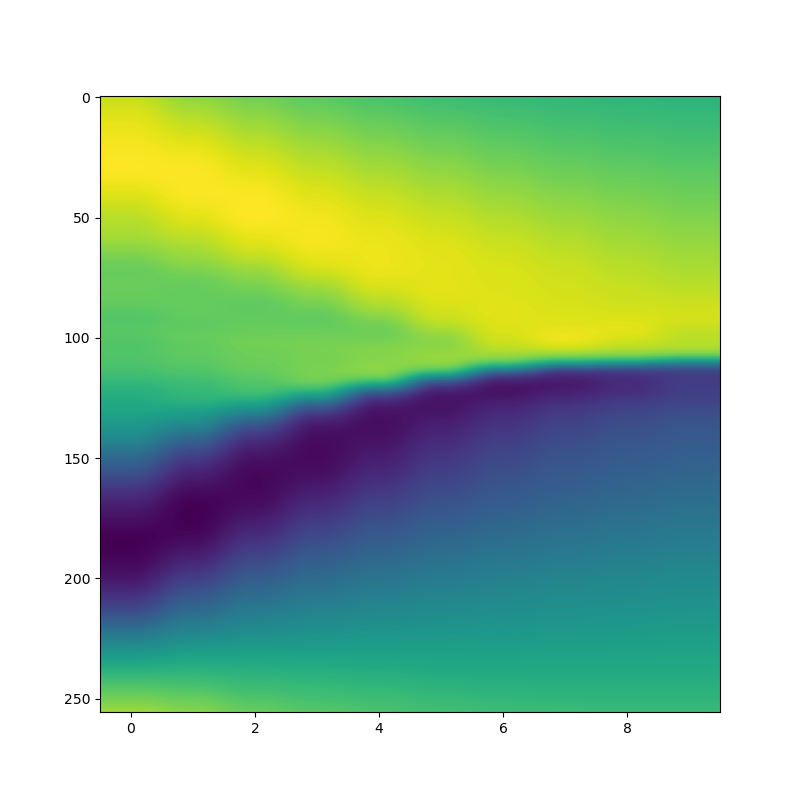}
			\end{center}
			\caption{Example of model's prediction for Burgers' dynamics where $\vec{x}$-axis is time and $\vec{y}$-axis represents space}
			\label{fig:pred}
		\end{figure}

		We notice that in the spirals tasks, FNO1D is initialized using more points ($5$ or $10$ intead of $2$ as in the Leray-Schauder neural operator). Therefore, FNO1D is given more points for initialization. For the Burgers' equation, FNO2D was seen to perform significantly worse than ANIE in \cite{ANIE}, and we have therefore compared only with the latter in the present article for a more concise treatment.

\end{document}